%% file: main.tex
\documentclass[10pt,letterpaper]{article}

\usepackage[letterpaper,text={5.5in,9in},centering,top=1in]{geometry}
\usepackage{mathtools,amsthm,amsfonts}
\usepackage{newtxtext,newtxmath}
\usepackage{microtype}
\usepackage{booktabs,tabularx,longtable,array,multirow}
\usepackage{graphicx}
\graphicspath{{figures/}}
\usepackage{caption,subcaption}
\usepackage{placeins}
\usepackage{float}
\usepackage{enumitem}
\usepackage{algorithm}
\usepackage{algpseudocode}
\usepackage{tikz}
\usetikzlibrary{arrows.meta,positioning,fit,calc,backgrounds,shapes.geometric,shadows.blur}
\usepackage{titlesec}
\usepackage{xcolor}
\usepackage{xurl}
\usepackage{hyperref}
\usepackage[nameinlink,capitalize,noabbrev]{cleveref}
\usepackage[backend=biber,style=authoryear-comp,maxcitenames=2,maxbibnames=99,
            giveninits=true,uniquename=false,doi=true,url=true,isbn=false]{biblatex}
\definecolor{linkblue}{HTML}{1A3E6E}

\hypersetup{
  colorlinks=true,
  linkcolor=linkblue,
  citecolor=linkblue,
  urlcolor=linkblue,
  pdfauthor={Ayush Mittal and Dhruv Gupta},
  pdftitle={Calibrated for Which Policy? Conformal Uncertainty for Fixed World Models Under Intervention},
  pdfsubject={Conformal uncertainty for fixed world models under intervention}
}

\titleformat{\section}{\normalfont\large\bfseries}{\thesection}{1em}{}
\titleformat{\subsection}{\normalfont\normalsize\bfseries}{\thesubsection}{1em}{}
\titleformat{\subsubsection}{\normalfont\normalsize\bfseries}{\thesubsubsection}{1em}{}
\titlespacing*{\section}{0pt}{2.0ex plus 1ex minus .2ex}{1.3ex plus .2ex}
\titlespacing*{\subsection}{0pt}{1.8ex plus 1ex minus .2ex}{1.0ex plus .2ex}
\titlespacing*{\subsubsection}{0pt}{1.6ex plus 1ex minus .2ex}{0.9ex plus .2ex}
\renewenvironment{abstract}{%
  \vskip 0.15in
  \begin{center}\normalfont\large\bfseries Abstract\end{center}
  \par\small
  \begingroup
  \leftskip=0.5in \rightskip=0.5in
  \noindent\ignorespaces
}{%
  \par\endgroup
  \vskip 0.15in
}

\setlist[itemize]{leftmargin=1.5em,itemsep=0.25em,topsep=0.35em}
\setlist[enumerate]{leftmargin=1.65em,itemsep=0.3em,topsep=0.4em}


\newcommand{\TV}{\operatorname{TV}}
\newcommand{\KL}{\operatorname{KL}}
\newcommand{\EE}{\mathbb E}
\newcommand{\PP}{\mathbb P}
\newcommand{\RR}{\mathbb R}
\newcommand{\ind}{\mathbf 1}
\newcommand{\calC}{\mathcal C}
\newcommand{\calX}{\mathcal X}
\newcommand{\calY}{\mathcal Y}
\newcommand{\calS}{\mathcal S}
\newcommand{\calA}{\mathcal A}
\newcommand{\Ber}{\operatorname{Bern}}
\newcommand{\ESS}{\operatorname{ESS}}
\newcommand{\dd}{\,\mathrm d}

\newtheorem{result}{Result}
\newtheorem{proposition}[result]{Proposition}
\newtheorem{lemma}[result]{Lemma}

\theoremstyle{definition}

\DeclareFieldFormat{url}{\mkbibacro{URL}\addcolon\space\url{#1}}
\DeclareFieldFormat{doi}{\mkbibacro{DOI}\addcolon\space\href{https://doi.org/#1}{#1}}
\AtEveryBibitem{\clearfield{urlyear}\clearfield{urlmonth}\clearfield{urlday}}

\title{\bfseries\LARGE Calibrated for Which Policy?\\[2.5mm]
  \large Conformal Uncertainty for Fixed World Models Under Intervention}
\author{%
\begin{tabular}{c@{\hspace{16mm}}c}
\normalsize Ayush Mittal & \normalsize Dhruv Gupta\\[1.2mm]
\small\href{mailto:ayush@noumenal.world}{\texttt{ayush@noumenal.world}} &
\small\href{mailto:dhruvgupta@iisc.ac.in}{\texttt{dhruvgupta@iisc.ac.in}}\\
& \small\href{mailto:dhruv.gupta@artpark.in}{\texttt{dhruv.gupta@artpark.in}}\\
& \small\href{mailto:dhruv@noumenal.world}{\texttt{dhruv@noumenal.world}}
\end{tabular}%
}
\date{}

\begin{document}
\maketitle

\begin{abstract}
Changing a deployment policy can invalidate a world model's uncertainty while leaving the physical
transition mechanism and predictor unchanged. For one-step prediction, if source and target laws share the conditional mechanism $P(Y\mid X)=Q(Y\mid X)$, then the intervention is
identified by the pre-outcome change of measure $w(X)=\dd Q_X/\dd P_X$. For policy-controlled
dynamics this ratio factors into state-occupancy and action-propensity terms. Weighted conformal
calibration transports marginal coverage to the target law; with an estimated ratio $\widehat w$,
\[
  \PP_Q\{Y\in\calC_{\widehat w}(X)\}
  \ge 1-\alpha-\TV(Q,\widehat Q)
  \ge 1-\alpha-\EE_P|w-\widehat w|.
\]
Population log-loss regret for a source-versus-target domain classifier bounds the same ratio error.
If $h$ is its clipped target posterior, $\rho$ is the target-domain prior, and
$\Delta_{\log}(h)$ is population log-loss regret, then
\[
  \EE_P|w-\widehat w_h|
  \le \frac{1}{\rho\gamma}\sqrt{\frac{\Delta_{\log}(h)}{2}}.
\]
An exact cancellation improves the dependence from a generic odds-map bound of order
$\gamma^{-2}$ to $(\rho\gamma)^{-1}$. Our experiments use a released DINO-feature world model and physical Franka
trajectories. Under the strongest controlled
state-action shift, ordinary 90\% conformal sets cover 55.0\%, whereas oracle and learned-ratio
calibration cover 91.0\% and 90.9\%. The learned surrogate remains within 0.9\% total variation of
the target law. Calibration is relative to the deployment law: the model can remain fixed while
calibration follows the intervention. We machine-check in Lean 4 the finite-support algebraic and causal identities used in
the classifier-regret argument, including its coverage-composition step.\footnotemark
\end{abstract}
\footnotetext{Our proofs are complete to the best of our knowledge, pass machine
verification, and use no project-specific axioms. Source:
\href{https://github.com/noumenal-ai/conformal-realrobot-lab/tree/8446f05a094d42752569751945c96be70666b701/lab/lean}{commit-pinned Lean project}.}

\section{Introduction}

A world model calibrated for the collection policy is not calibrated for the deployment policy.
A model trained under one pattern of action selection is assessed under another once a
planner, controller, or downstream policy starts choosing actions from its predictions. The
transition mechanism may remain stable, yet the distribution of state--action queries can move
into regions where the model's residuals follow a different marginal law. Calibration under the
collection policy is then a statement about the wrong measure. Shifting deployment mass toward
state--action pairs with larger residuals lowers the target coverage of unweighted source quantiles.

World models have moved from compact generative simulators and latent planning
\parencite{ha2018worldmodels,hafner2019planet,chua2018pets} to policy learning through imagined
trajectories \parencite{hafner2020dreamer,hafner2025dreamerv3}, token-based models
\parencite{micheli2023iris}, and scalable latent control \parencite{hansen2024tdmpc2}. DayDreamer
learns robot behavior through a learned
latent dynamics model \parencite{wu2023daydreamer}; DINO-WM plans by predicting action-conditioned
DINO features \parencite{zhou2025dinowm,oquab2023dinov2}; V-JEPA 2-AC and JEPA-WMs extend
predictive representations to real-robot planning
\parencite{assran2025vjepa2,terver2025jepawms}; and WorldGym evaluates policies inside a learned
world model \parencite{quevedo2025worldgym}. These systems expand what a fixed model is asked to
predict. The policy that generates deployment queries becomes part of the uncertainty problem.

Prediction sets from conformal inference provide finite-sample marginal coverage under
exchangeability \parencite{vovk2005algorithmic,lei2018distributionfree}. Weighted conformal
prediction extends this guarantee when source and target covariates differ and their likelihood
ratio is available \parencite{tibshirani2019covariateshift}; later work studies broader departures
from exchangeability and adaptive drift \parencite{barber2023beyond,gibbs2021adaptive}. Propensity
weighting also enters conformal counterfactual inference for treatment assignments
\parencite{lei2021counterfactual}. These results provide the statistical transport mechanism. A
policy intervention in a dynamical system additionally requires identification of its density ratio
and a coverage bound for ratio-estimation error. Once the ratio is identified, oracle weighting
restores nominal target-law coverage without changing the predictor.

Deployment changes the law under which a fixed world model is evaluated. Let $X$ collect information
available before the predicted outcome $Y$, and let a fixed world model induce a score
$r(X,Y)$. If an intervention preserves the conditional mechanism,
\[
  P(Y\mid X)=Q(Y\mid X),
\]
then its joint likelihood ratio is $X$-measurable:
\[
  \frac{\dd Q}{\dd P}(X,Y)=\frac{\dd Q_X}{\dd P_X}(X).
\]
For a policy-controlled transition, this ratio factors into a state-occupancy correction and an
action-propensity correction. Under a shared pre-action state law, the occupancy term is one and
the action propensity is the complete ratio. The causal structure identifies what weighted
conformal calibration must estimate.

We show that approximate identification yields a direct coverage penalty. An estimated nonnegative ratio
$\widehat w$ defines a normalized surrogate target $\widehat Q$. Weighted conformal prediction is
valid under that surrogate, and total variation transfers the event probability to the intended
intervention. When $\widehat Q$ stays close to $Q$, learned weighting follows the oracle, with
realized coverage above $1-\alpha-\TV(Q,\widehat Q)$. Clipping replaces $\widehat Q$ with another
surrogate; an effective-sample-size gain is accompanied by coverage loss when clipping increases
$\TV(Q,\widehat Q)$. We use a domain classifier to supply the ratio from unlabeled source and target
covariates. We show that its population log-loss regret controls ratio error through an exact
change-of-measure identity, which in turn controls coverage. A conservative global posterior clipping floor $\gamma$
makes the classifier-regret bound looser than the exact-TV bound.

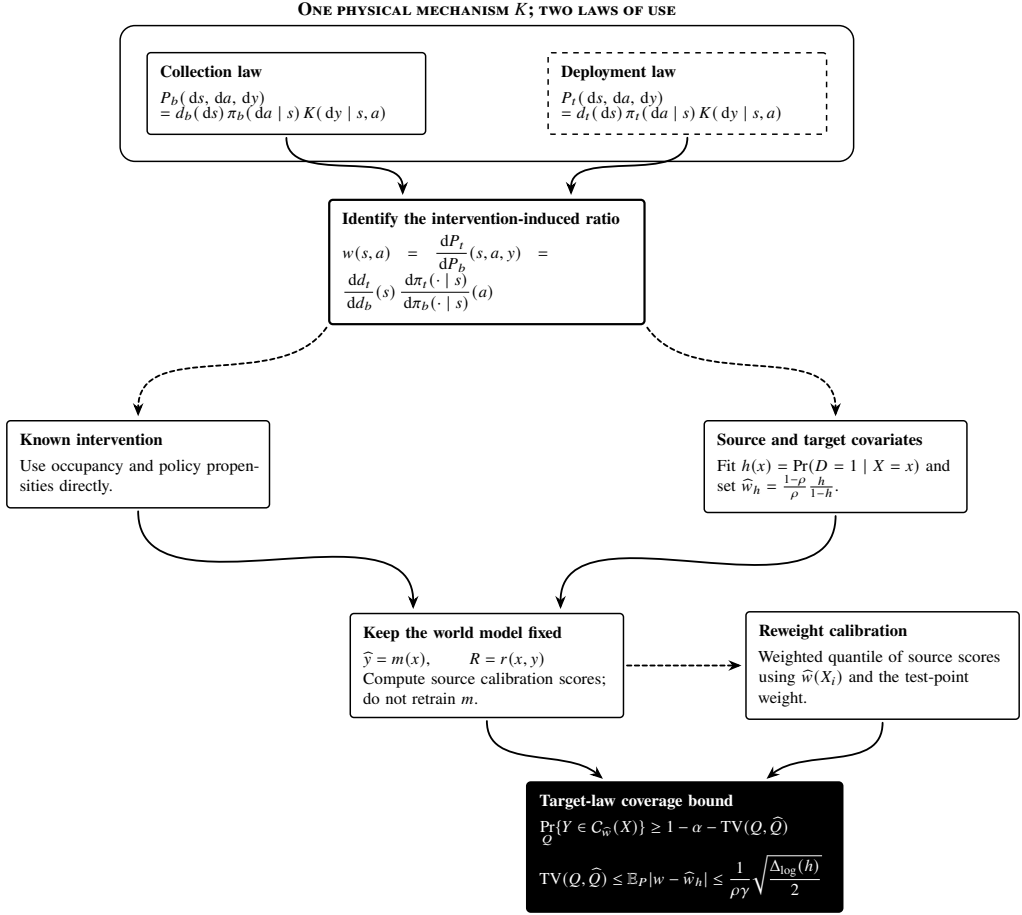
\begin{figure}[!htbp]
\centering
\resizebox{0.96\textwidth}{!}{%
\begin{tikzpicture}[
  font=\normalfont\small,
  >=Stealth,
  % Figure 1 is set in black and white only; the source/target contrast is carried
  % by rule style rather than hue, and the certificate panel by reversed type.
  box/.style={draw=black, rounded corners=2pt, thick, align=left,
              inner sep=7pt, minimum height=15mm, fill=white},
  source/.style={box, text width=48mm},
  target/.style={box, text width=48mm, dashed},
  bridge/.style={box, text width=55mm, very thick},
  method/.style={box, text width=45mm},
  cert/.style={box, fill=black, text=white, draw=black, text width=55mm, very thick},
  arrow/.style={-{Stealth[length=2.5mm,width=2mm]}, line width=1pt,
                color=black, line cap=round, line join=round,
                shorten <=1.5pt, shorten >=2pt},
  altarrow/.style={-{Stealth[length=2.5mm,width=2mm]}, line width=1pt,
                   color=black, dash pattern=on 2.4pt off 1.8pt,
                   line cap=round, line join=round,
                   shorten <=1.5pt, shorten >=2pt},
]

\node[source] (src) {
  \textbf{Collection law}\par\smallskip
  $P_b(\dd s,\dd a,\dd y)$\\[-1mm]
  $=d_b(\dd s)\,\pi_b(\dd a\mid s)\,K(\dd y\mid s,a)$
};
\node[target, right=23mm of src] (tgt) {
  \textbf{Deployment law}\par\smallskip
  $P_t(\dd s,\dd a,\dd y)$\\[-1mm]
  $=d_t(\dd s)\,\pi_t(\dd a\mid s)\,K(\dd y\mid s,a)$
};
\node[draw=black, rounded corners=8pt, thick, fit=(src)(tgt), inner sep=5mm,
      label={[font=\normalfont\bfseries\scshape\color{black}]above:One physical mechanism $K$; two laws of use}] (laws) {};

\node[bridge, below=20mm of $(src)!0.5!(tgt)$] (identify) {
  \textbf{Identify the intervention-induced ratio}\par\smallskip
  $\displaystyle w(s,a)=\frac{\dd P_t}{\dd P_b}(s,a,y)
    =\frac{\dd d_t}{\dd d_b}(s)\,
      \frac{\dd\pi_t(\cdot\mid s)}{\dd\pi_b(\cdot\mid s)}(a)$
};
\draw[arrow] (src.south) to[out=270,in=90] ($(identify.north)+(-16mm,0)$);
\draw[arrow] (tgt.south) to[out=270,in=90] ($(identify.north)+(16mm,0)$);

\node[method, below left=18mm and 11mm of identify] (known) {
  \textbf{Known intervention}\par\smallskip
  Use occupancy and policy propensities directly.
};
\node[method, below right=18mm and 11mm of identify] (learned) {
  \textbf{Source and target covariates}\par\smallskip
  Fit $h(x)=\Pr(D=1\mid X=x)$ and set
  $\widehat w_h=\frac{1-\rho}{\rho}\frac{h}{1-h}$.
};
\draw[altarrow] (identify.south west) to[out=235,in=90] (known.north);
\draw[altarrow] (identify.south east) to[out=305,in=90] (learned.north);

\node[box, text width=47mm, below=28mm of $(known)!0.5!(learned)$] (score) {
  \textbf{Keep the world model fixed}\par\smallskip
  $\widehat y=m(x)$, \qquad $R=r(x,y)$\\
  Compute source calibration scores; do not retrain $m$.
};
\draw[arrow] (known.south) to[out=270,in=90] ($(score.north)+(-14mm,0)$);
\draw[arrow] (learned.south) to[out=270,in=90] ($(score.north)+(14mm,0)$);

\node[method, right=23mm of score, text width=47mm] (calibrate) {
  \textbf{Reweight calibration}\par\smallskip
  Weighted quantile of source scores using $\widehat w(X_i)$ and the test-point weight.
};
\draw[altarrow] (score) -- (calibrate);

\node[cert, below=22mm of $(score)!0.5!(calibrate)$] (bound) {
  \textbf{Target-law coverage bound}\par\smallskip
  $\displaystyle
  \Pr_Q\{Y\in\mathcal C_{\widehat w}(X)\}
  \ge 1-\alpha-\TV(Q,\widehat Q)$\\[1mm]
  $\displaystyle
  \TV(Q,\widehat Q)
  \le \EE_P|w-\widehat w_h|
  \le \frac{1}{\rho\gamma}\sqrt{\frac{\Delta_{\log}(h)}{2}}$
};
\draw[arrow] (score.south) to[out=270,in=140] ($(bound.north)+(-15mm,0)$);
\draw[arrow] (calibrate.south) to[out=270,in=40] ($(bound.north)+(15mm,0)$);

\end{tikzpicture}}
\caption{The intervention changes the measure under which a fixed world model is assessed. Shared
conditional dynamics make the target/source ratio a function of pre-outcome variables. That ratio
can come from known policy structure or from domain classification. Weighted conformal calibration
then transports the score law to the deployment measure, with a penalty determined by ratio error.}
\label{fig:concept}
\end{figure}
\FloatBarrier

\section{Policy shift identifies a change of measure}
\label{sec:identification}

Let $X\in\calX$ denote information available before the outcome $Y\in\calY$. For a one-step
world-model query, $X=(S,A)$ contains the current state representation and chosen action, while $Y$
is the next observation, state, or latent target. A fixed predictor $m$ and a measurable
nonconformity score $r_m:\calX\times\calY\to\RR$ define the residual whose distribution is
calibrated.

Source data follow $P$ and deployment follows $Q$. An intervention can change the distribution of
$X$ while preserving the conditional mechanism for $Y$, or it can change the mechanism itself.

\begin{proposition}[Pre-outcome identification]
\label{prop:preoutcome}
Suppose $Q_X\ll P_X$ and there exists a Markov kernel $K$ such that
\[
  P(\dd x,\dd y)=P_X(\dd x)K(\dd y\mid x),\qquad
  Q(\dd x,\dd y)=Q_X(\dd x)K(\dd y\mid x).
\]
Then $Q\ll P$ and
\[
  \frac{\dd Q}{\dd P}(x,y)=\frac{\dd Q_X}{\dd P_X}(x)
  \quad P\text{-a.e.}
\]
\end{proposition}

\begin{proof}
Set $w_X=\dd Q_X/\dd P_X$. For bounded measurable $f$, we compute
\[
\begin{aligned}
  \int f(x,y)\,Q(\dd x,\dd y)
  &=\int\!\left[\int f(x,y)K(\dd y\mid x)\right]Q_X(\dd x)\\
  &=\int w_X(x)\left[\int f(x,y)K(\dd y\mid x)\right]P_X(\dd x)\\
  &=\int w_X(x)f(x,y)\,P(\dd x,\dd y).
\end{aligned}
\]
We conclude that the Radon--Nikodym derivative is $w_X\circ\operatorname{proj}_X$.
\end{proof}

The outcome $Y$ is unavailable when a prediction is issued, so we must fit the ratio from
pre-outcome variables. Those variables must render the environment mechanism stable.

\subsection{Policy-controlled dynamics}

Let $S\in\calS$ be a state representation, $A\in\calA$ an action, and $Y$ the next target. A
behavior policy $\pi_b$ and deployment policy $\pi_t$ induce
\begin{align}
  P_b(\dd s,\dd a,\dd y)
  &=d_b(\dd s)\,\pi_b(\dd a\mid s)\,K(\dd y\mid s,a),\label{eq:pb}\\
  P_t(\dd s,\dd a,\dd y)
  &=d_t(\dd s)\,\pi_t(\dd a\mid s)\,K(\dd y\mid s,a).\label{eq:pt}
\end{align}
Here $d_b$ and $d_t$ are the state occupancies at the query time. The action mechanism is replaced
while the transition kernel is shared, matching a soft intervention on action selection
\parencite{pearl2009causality}.

\begin{proposition}[Policy-intervention ratio]
\label{prop:policy-ratio}
If $d_t\ll d_b$ and $\pi_t(\cdot\mid s)\ll\pi_b(\cdot\mid s)$ for $d_t$-almost all $s$, then
\[
  \frac{\dd P_t}{\dd P_b}(s,a,y)
  =\frac{\dd d_t}{\dd d_b}(s)
   \frac{\dd\pi_t(\cdot\mid s)}{\dd\pi_b(\cdot\mid s)}(a).
\]
\end{proposition}

\begin{proof}
We apply Proposition~\ref{prop:preoutcome} with $X=(S,A)$. The state--action marginal factorizes into occupancy
and conditional policy, so its derivative is the product of the two conditional derivatives.
\end{proof}

A one-step intervention from a declared state population $\mu$ sets $d_b=d_t=\mu$. The ratio
reduces to
\[
  w(s,a)=\frac{\dd\pi_t(\cdot\mid s)}{\dd\pi_b(\cdot\mid s)}(a).
\]
For later stages of deployment, action selection changes the state occupancy and the factor
$\dd d_t/\dd d_b$ remains. Stationary-distribution correction methods in off-policy learning target
this object directly \parencite{nachum2019dualdice,liu2018horizon}. Trajectory-level importance
ratios multiply policy ratios across time and can deteriorate with horizon
\parencite{jiang2016doubly}. For one-step marginal prediction, the identified query law determines
the calibration measure.

Under these conditions, we can treat policy shift as a post-hoc calibration problem.

\section{Transporting conformal coverage}
\label{sec:coverage}

Let calibration pairs $Z_i=(X_i,Y_i)\overset{\mathrm{iid}}\sim P$ produce scores
$R_i=r_m(X_i,Y_i)$. For a nonnegative weight function $v$ and test covariate $x$, we define the
weighted empirical measure
\begin{equation}
  \sum_{i=1}^{n}
  \frac{v(X_i)}{\sum_{j=1}^{n}v(X_j)+v(x)}\,\delta_{R_i}
  +
  \frac{v(x)}{\sum_{j=1}^{n}v(X_j)+v(x)}\,\delta_{+\infty}.
  \label{eq:weighted-measure}
\end{equation}
Its $(1-\alpha)$ quantile is $q_v(x)$ and we take
\[
  \calC_v(x)=\{y:r_m(x,y)\le q_v(x)\}.
\]
The test-point mass in \cref{eq:weighted-measure} gives the finite-sample weighted conformal
construction under covariate shift \parencite{tibshirani2019covariateshift}.

If $w=\dd Q/\dd P$ is $X$-measurable, weighted conformal prediction with $w$ attains target-law
marginal coverage at least $1-\alpha$.

Let $\widehat w:\calX\to[0,\infty)$ be fixed or fitted on data independent of calibration, and set
\begin{equation}
  Z_{\widehat w}=\EE_P[\widehat w(X)],
  \qquad
  \frac{\dd\widehat Q}{\dd P}(x,y)
  =\frac{\widehat w(x)}{Z_{\widehat w}}.
  \label{eq:qhat}
\end{equation}
Global scaling does not change \cref{eq:weighted-measure}; hence $\widehat w$ is the exact
likelihood ratio for the surrogate target $\widehat Q$.

\begin{proposition}[Approximate-weight coverage transfer]
\label{prop:approx-coverage}
Assume $Q\ll P$, $w=\dd Q/\dd P$ is $X$-measurable, $0<Z_{\widehat w}<\infty$, and the fitted
weight is independent of the calibration sample. Then
\begin{equation}
  \PP_{P^n\otimes Q}\{Y_{n+1}\in\calC_{\widehat w}(X_{n+1})\}
  \ge 1-\alpha-\TV(Q,\widehat Q).
  \label{eq:tv-transfer}
\end{equation}
Moreover,
\begin{equation}
  \TV(Q,\widehat Q)
  \le \min\left\{1,\EE_P|w-\widehat w|\right\},
  \label{eq:l1-tv}
\end{equation}
so coverage is at least $1-\alpha-\EE_P|w-\widehat w|$.
\end{proposition}

\begin{proof}
We condition on $\widehat w$. Weighted conformal validity under $P^n\otimes\widehat Q$ gives
coverage at least $1-\alpha$. For a fixed calibration sample, the coverage event is a measurable subset of
the test space, so replacing $\widehat Q$ by $Q$ reduces its probability by at most
$\TV(Q,\widehat Q)$. Averaging over calibration proves \cref{eq:tv-transfer}.

For \cref{eq:l1-tv}, we write $\varepsilon_1=\EE_P|w-\widehat w|$ and note that
$|Z_{\widehat w}-1|\le\varepsilon_1$. Then
\[
\begin{aligned}
2\TV(Q,\widehat Q)
&=\EE_P\left|w-\frac{\widehat w}{Z_{\widehat w}}\right|\\
&\le \EE_P|w-\widehat w|
 +\EE_P\left|\widehat w-\frac{\widehat w}{Z_{\widehat w}}\right|\\
&=\varepsilon_1+|Z_{\widehat w}-1|
\le2\varepsilon_1.
\end{aligned}
\]
\end{proof}

Total variation is the exact distribution-only penalty for an arbitrary event. The
$L^1(P)$ ratio error supplies an estimator-facing interface. Large or concentrated weights can
inflate the conformal radius or drive the test-dependent quantile to $+\infty$, but they do not
alter the coverage statement. Efficiency and validity therefore separate: clipping can improve
$\ESS=(\sum_iw_i)^2/\sum_iw_i^2$ while moving the surrogate away from the intended target.

\section{Learning the change of measure}
\label{sec:classifier}

Source and target covariates can identify the ratio without target outcomes. Draw a domain label
$D\sim\Ber(\rho)$, followed by
\[
  X\mid D=0\sim P_X,
  \qquad
  X\mid D=1\sim Q_X.
\]
The mixture marginal is $M=(1-\rho)P_X+\rho Q_X$, and
$\eta(x)=\Pr(D=1\mid X=x)$ is the Bayes domain posterior. Bayes' rule yields
\begin{equation}
  w(x)=\frac{\dd Q_X}{\dd P_X}(x)
  =\frac{1-\rho}{\rho}\frac{\eta(x)}{1-\eta(x)}.
  \label{eq:posterior-odds}
\end{equation}
Let $h:\calX\to[\gamma,1-\gamma]$ be a fitted target-domain posterior and define
\[
  \widehat w_h(x)=\frac{1-\rho}{\rho}\frac{h(x)}{1-h(x)}.
\]
Its population binary log-loss regret is
\begin{equation}
  \Delta_{\log}(h)
  =R(h)-R(\eta)
  =\EE_M\!
  \left[\KL\bigl(\Ber(\eta(X))\,\|\,\Ber(h(X))\bigr)\right].
  \label{eq:regret}
\end{equation}

\begin{lemma}[Posterior-to-ratio identity]
\label{lem:exact-cancel}
Under the preceding setup,
\begin{equation}
  \EE_{P_X}|w-\widehat w_h|
  =\frac{1}{\rho}\EE_M\!\left[\frac{|\eta-h|}{1-h}\right].
  \label{eq:exact-cancel}
\end{equation}
\end{lemma}

\begin{proof}
By Bayes' rule, $\dd P_X/\dd M=(1-\eta)/(1-\rho)$ and
$\dd Q_X/\dd M=\eta/\rho$. Since
\[
  \left|\frac{\eta}{1-\eta}-\frac{h}{1-h}\right|
  =\frac{|\eta-h|}{(1-\eta)(1-h)},
\]
we change measure from $P_X$ to $M$; the factor $1-\eta$ cancels and we obtain
\cref{eq:exact-cancel}.
\end{proof}

A generic Lipschitz bound on posterior odds would introduce two denominator factors. Our exact
identity removes one because we integrate ratio error under the source law rather than the domain
mixture.

\begin{proposition}[Classifier regret controls ratio error]
\label{prop:classifier}
For $h(x)\in[\gamma,1-\gamma]$ and $\rho\in(0,1)$,
\begin{equation}
  \EE_{P_X}|w-\widehat w_h|
  \le
  \frac{1}{\rho\gamma}
  \sqrt{\frac{\Delta_{\log}(h)}{2}}.
  \label{eq:classifier-bound}
\end{equation}
For balanced domain sampling, $\rho=1/2$,
\begin{equation}
  \EE_{P_X}|w-\widehat w_h|
  \le \frac{\sqrt{2\Delta_{\log}(h)}}{\gamma}.
  \label{eq:balanced-bound}
\end{equation}
\end{proposition}

\begin{proof}
From Lemma~\ref{lem:exact-cancel} and $1-h\ge\gamma$, we obtain
\[
  \EE_{P_X}|w-\widehat w_h|
  \le \frac{1}{\rho\gamma}\EE_M|\eta-h|.
\]
Binary Pinsker gives
$2(\eta-h)^2\le
\KL(\Ber(\eta)\|\Ber(h))$ pointwise. Applying Cauchy--Schwarz and \cref{eq:regret} yields
\[
  \EE_M|\eta-h|
  \le\sqrt{\EE_M(\eta-h)^2}
  \le\sqrt{\frac{\Delta_{\log}(h)}{2}}.
\]
Substitution proves \cref{eq:classifier-bound}; setting $\rho=1/2$ gives
\cref{eq:balanced-bound}.
\end{proof}

\begin{result}[End-to-end target coverage]
\label{res:end-to-end}
Under the assumptions of Propositions~\ref{prop:approx-coverage} and~\ref{prop:classifier},
\begin{equation}
  \PP_Q\{Y\in\calC_{\widehat w_h}(X)\}
  \ge
  1-\alpha-
  \frac{1}{\rho\gamma}
  \sqrt{\frac{\Delta_{\log}(h)}{2}}.
  \label{eq:end-to-end}
\end{equation}
\end{result}

The bound uses population log-loss regret. A finite-support deployment study can compute
\cref{eq:regret} by direct summation when $P_X$ and $Q_X$ are specified over an observed pool. In a
general deployment, a held-out or uniform-convergence upper bound on regret can be substituted.
The exact surrogate penalty $\TV(Q,\widehat Q)$ is sharper when it can be estimated reliably.
Transport proximity alone does not identify the reweighting ratio: in
\cref{app:transport-separation} we construct compactly supported, bounded-density laws with
$W_2\to0$ while total variation and $\chi^2$ remain fixed.

\section{Physical robot experiment}
\label{sec:experiment}

\subsection{Frozen external model and physical outcomes}

We use the released JEPA-WMs implementation and its physical Franka evaluation data
\parencite{terver2025jepawms,jepawms_repo}. We take as predictor the released \texttt{dino\_wm\_droid} checkpoint: a
DINO-WM that predicts future DINOv2 patch features from visual context and robot actions
\parencite{zhou2025dinowm,oquab2023dinov2,dinov2_repo}. Our evaluation pool contains 192 eligible
one-step transitions from 15 physical Franka episodes spanning picking, pushing, and folding
tasks. Related DROID data underlie the broader real-robot training substrate used by these model
families \parencite{khazatsky2024droid}.

We form $X$ from two observed context frames at 4 Hz and their aligned seven-dimensional
pose-difference actions. The target $Y$ is the encoder representation of the observed next frame. We take the score to be root-mean-squared error between the action-conditioned predicted
latent and the frozen encoder representation of that physical next observation. We hold the released predictor fixed; only the calibration measure changes.

\subsection{Controlled target laws on the observed pool}

Let $\Omega$ index the 192 transitions. We define $g_i\in[-1,1]$ as the empirical midrank of the
translational action magnitude for transition $i$. The statistic is determined before the target
frame and does not inspect the residual. For $\lambda\in\{0,0.5,1,1.5\}$,
\begin{equation}
  P_\lambda(i)=
  \frac{e^{-\lambda g_i}}{\sum_{j\in\Omega}e^{-\lambda g_j}},
  \qquad
  Q_\lambda(i)=
  \frac{e^{+\lambda g_i}}{\sum_{j\in\Omega}e^{+\lambda g_j}}.
  \label{eq:tilt}
\end{equation}
The support is unchanged. Conditional on an index, the physical target and frozen-model score are
unchanged. The source and target laws therefore share the conditional score mechanism, and we can write
the oracle ratio as
\[
  w_\lambda(i)=\frac{Q_\lambda(i)}{P_\lambda(i)}.
\]
At $\lambda=1.5$, this ratio ranges from $0.0498$ to $20.09$, and
$\TV(P_\lambda,Q_\lambda)=0.637$.

We draw a balanced domain sample independently of conformal calibration and fit $h(i)=\Pr(D=1\mid g_i)$
by logistic regression on the preregistered feature $g$; we restrict posterior values to
$[0.05,0.95]$ and convert them to odds. The population model is correctly specified because
\[
  \log\frac{Q_\lambda(i)}{P_\lambda(i)}=c_\lambda+2\lambda g_i.
\]
We compute the oracle ratio, Bayes posterior, population classifier regret, $\EE_P|w-\widehat w|$,
and $\TV(Q,\widehat Q)$ by finite summation.

For calibration we draw 500 source transitions with replacement; for target evaluation we
draw 10,000 transitions, and we repeat the full protocol fifty times for each
$\alpha\in\{0.05,0.10,0.20\}$. We compare four methods: ordinary source conformal calibration,
oracle ratio weighting, learned ratio weighting, and learned ratio weighting clipped at five. We
report Student-$t$ intervals across repetitions; episode-cluster bootstrap intervals appear in
the appendix.

\section{Results}
\label{sec:results}

\subsection{Controlled covariate shift breaks ordinary calibration}

The target score law moves as $\lambda$ increases. At nominal 90\% coverage, ordinary source
calibration falls from 90.2\% at $\lambda=0$ to 81.6\%, 69.0\%, and 55.0\%. Oracle weighting
returns 90.5\%, 91.8\%, and 91.0\% over the same nonzero shifts. Learned weighting gives 90.6\%,
91.8\%, and 90.9\%. We hold the predictor identical across methods; the difference is the calibration
measure.

\begin{figure}[!htbp]
  \centering
  \includegraphics[width=\textwidth]{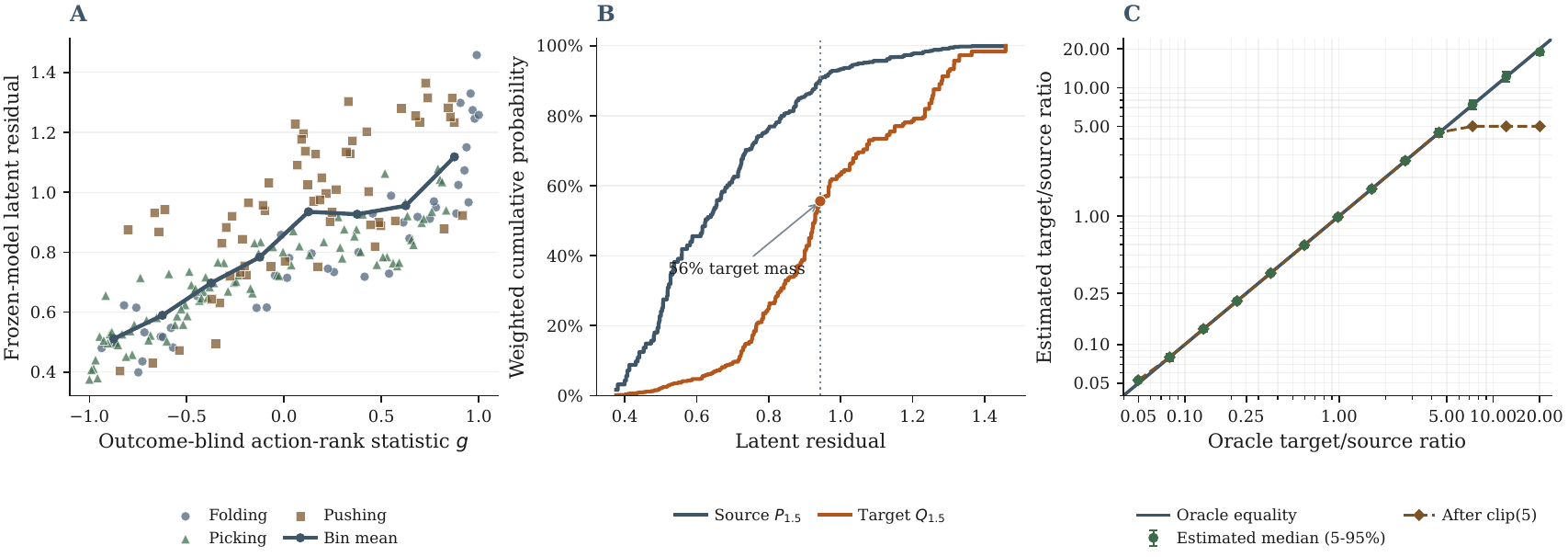}
  \caption{The controlled shift acts on physical robot transitions. \textbf{A}: the outcome-blind
  action statistic is strongly associated with the frozen model's latent residual, so changing its
  mass changes the score law. \textbf{B}: the source 90th-percentile score contains 55--56\% of
  target mass. \textbf{C}: at observed pool points, markers show the median learned ratio and
  whiskers show the 5--95\% interval across 50 domain-classifier fits; the diagonal is oracle
  equality and clip(5) truncates the target tail.}
  \label{fig:shift-geometry}
\end{figure}

\begin{figure}[!htbp]
  \centering
  \includegraphics[width=\textwidth]{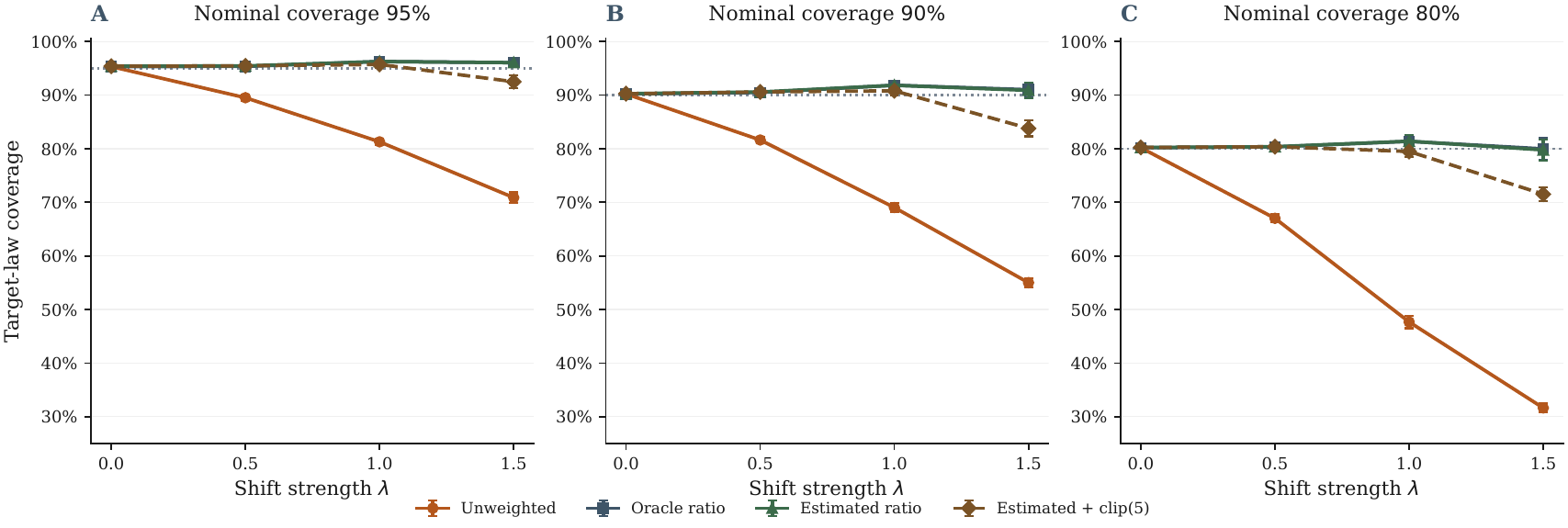}
  \caption{Target-law coverage across the preregistered levels. Ordinary conformal calibration
  loses coverage as the source and target measures separate. Oracle and learned ratios restore
  the target-law guarantee at nominal 95\%, 90\%, and 80\% coverage. Error bars are 95\%
  Student-$t$ intervals across 50 repetitions.}
  \label{fig:coverage}
\end{figure}

\begin{table}[!htbp]
  \centering
  \caption{Coverage, effective sample size, and radius at nominal 90\% coverage. Ordinary
  calibration loses coverage as shift strength grows, while learned-ratio calibration tracks the
  oracle. Effective sample size and radius refer to the learned-ratio method.}
  \label{tab:headline}
  \small
  \resizebox{\textwidth}{!}{\input{table_headline.tex}}
\end{table}

At $\lambda=1.5$, learned weighting recovers 35.9 percentage points relative to unweighted
calibration. The mean radius rises from 1.183 at $\lambda=0$ to 1.281, while the infinite-radius
rate remains zero.

\subsection{The learned ratio reproduces the oracle correction}

The population domain-logit slopes are $2\lambda\in\{0,1,2,3\}$. Fitted mean slopes are 0.0016,
1.0065, 2.0012, and 3.0052. Across the grid, population log-loss regret remains between
$5.3\times10^{-5}$ and $8.3\times10^{-5}$ for the learned-ratio method. The resulting normalized
surrogate is 0.7--0.9\% from the target in total variation.

\begin{figure}[!htbp]
  \centering
  \includegraphics[width=\textwidth]{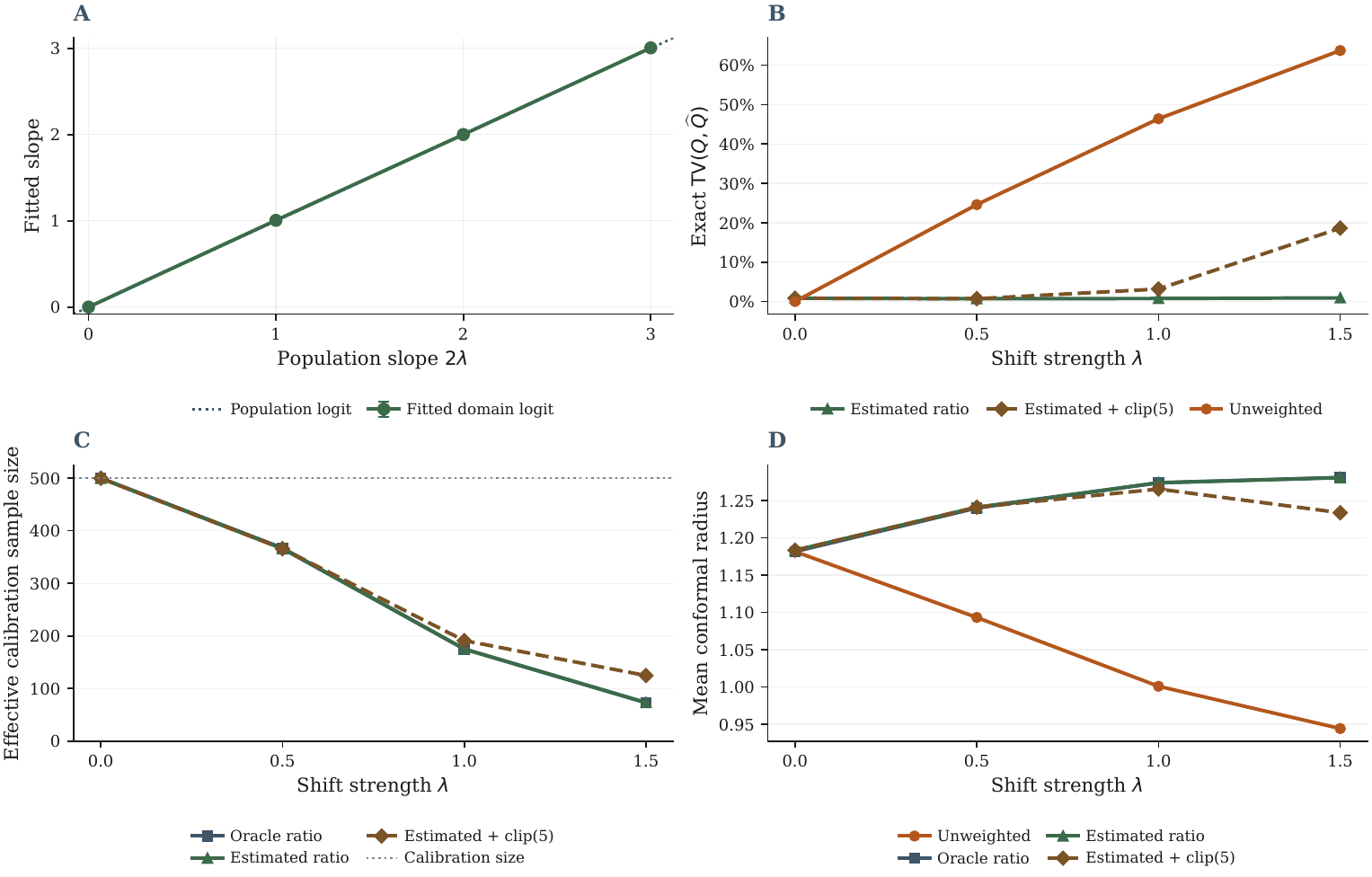}
  \caption{Ratio estimation and efficiency. \textbf{A}: fitted domain-classifier slopes recover
  the population log-odds. \textbf{B}: exact total variation between the target and induced
  surrogate remains below 1\% for learned weights; clipping departs at stronger shift.
  \textbf{C}: effective calibration size falls as importance weights concentrate.
  \textbf{D}: conformal radii widen in response.}
  \label{fig:estimation}
\end{figure}

Effective sample size falls from 500 to 73 as $\lambda$ reaches 1.5. Oracle and learned weighting
retain target coverage; weight concentration appears in wider radii and repetition intervals,
whereas surrogate error controls the coverage penalty.
\FloatBarrier

\subsection{Surrogate mismatch explains clipping failure}

For learned weights at nominal 90\%, the exact-TV lower bounds are 89.2\%, 89.3\%, 89.2\%, and
89.1\%. Observed coverage lies above each bound. The classifier-regret bounds are 70.8\%, 72.4\%,
68.7\%, and 66.5\%. Their looseness comes from the global denominator bound $1-h\ge\gamma$;
most posterior values lie farther from one.

\begin{figure}[H]
  \centering
  \includegraphics[width=\textwidth]{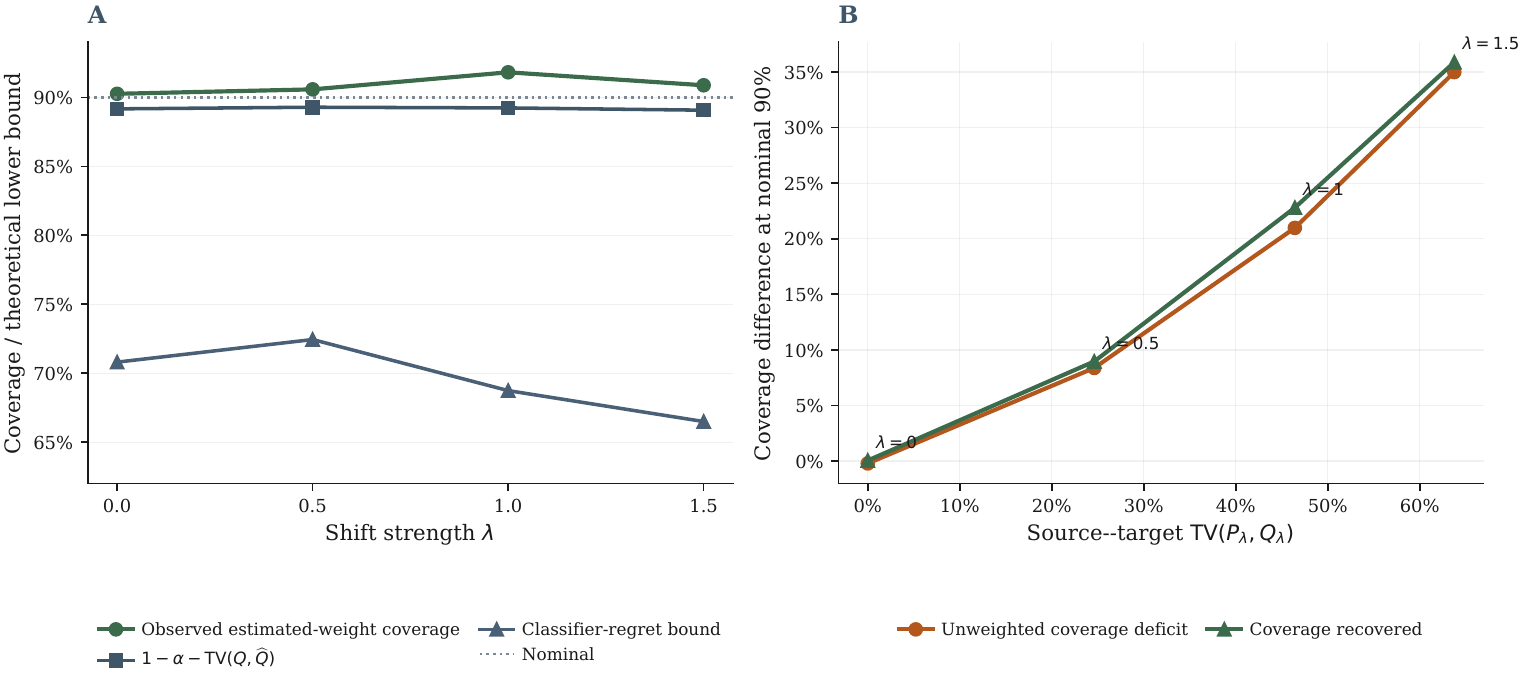}
  \caption{Coverage and theoretical lower bounds at nominal 90\% coverage. \textbf{A}: learned-weight coverage,
  the exact-TV bound, and the classifier-regret bound. \textbf{B}: ordinary coverage
  deficit and coverage recovered by learned weighting as source--target total variation grows.}
  \label{fig:diagnostics}
\end{figure}

\begin{table}[H]
  \centering
  \caption{Observed learned-weight coverage and the two theoretical lower bounds at nominal
  90\%. Population log-loss regret is computed exactly over the physical transition pool.}
  \label{tab:bounds}
  \small
  \input{table_bounds.tex}
\end{table}

Clipping at five leaves the method unchanged through $\lambda=0.5$, then moves its surrogate away
from the target. At $\lambda=1.5$, clipping raises effective sample size from 73 to 125 but raises
$\TV(Q,\widehat Q)$ from 0.9\% to 18.7\%; coverage falls from 90.9\% to 83.8\%. Clipping increases
effective sample size by changing the law being calibrated.
\FloatBarrier

\section{Relation to prior work}
\label{sec:related}

\paragraph{World models.}
Latent dynamics, stochastic state-space models, predictive representations, and token dynamics
have been developed to support planning and policy learning
\parencite{ha2018worldmodels,hafner2019planet,hafner2020dreamer,micheli2023iris,hansen2024tdmpc2}.
Physical-world applications include robot learning \parencite{wu2023daydreamer}, zero-shot
planning \parencite{zhou2025dinowm,assran2025vjepa2,terver2025jepawms}, and policy evaluation
\parencite{quevedo2025worldgym}. Adjacent lines adapt a policy inside a learned model
\parencite{chandra2025diwa} and establish that general goal-directed agents must have learned a
predictive model \parencite{richens2025worldmodels}. After action selection changes, calibration
concerns the distribution under which the fixed model's residual is evaluated.

\paragraph{Conformal inference under shift.}
Split conformal inference provides distribution-free marginal prediction under exchangeability
\parencite{vovk2005algorithmic,lei2018distributionfree}. Weighted conformal prediction handles
covariate shift through a known or estimated likelihood ratio
\parencite{tibshirani2019covariateshift}; broader weighting and adaptive schemes handle other
forms of nonexchangeability \parencite{barber2023beyond,gibbs2021adaptive}. Robust validation
bounds predictive performance under distribution shift \parencite{cauchois2024robustvalidation},
while conformal methods have been applied to controllers and trajectory forecasts
\parencite{chee2024controllers,zhou2024trajectories}. Latent safety filters also use conformal
thresholds to detect uncertain world-model predictions \parencite{seo2025latentsafety}.
Intervention identification supplies the ratio for weighted calibration; ratio-estimation error
determines the target-law coverage penalty.

\paragraph{Causal and off-policy correction.}
Propensity-weighted conformal inference has been developed for static treatment counterfactuals
\parencite{lei2021counterfactual}. Policy interventions introduce a state occupancy as well as an
action propensity. Off-policy evaluation and stationary-distribution correction study these
ratios for returns \parencite{jiang2016doubly,liu2018horizon,nachum2019dualdice}. The same
change-of-measure structure governs marginal coverage of one-step world-model prediction sets.

\paragraph{Density-ratio estimation.}
Covariate-shift adaptation can estimate ratios directly or through source--target classification
\parencite{sugiyama2007direct,bickel2009covariateshift}. Proper-loss links between class
probabilities and density ratios are developed by \textcite{menon2016linking}. Integrating ratio
error under the source law, we obtain the $(\rho\gamma)^{-1}$ constant in
Lemma~\ref{lem:exact-cancel}.

\section{Discussion}
\label{sec:discussion}

Training a world model attempts to learn $K(\dd y\mid x)$. Deploying a policy chooses the law of
$X$. A new policy need not force a new world model when the learned mechanism remains useful; it
does force a new statement
about which queries receive calibrated coverage. Occupancy and propensity identify that statement.
Unlabeled target covariates can estimate it before the corresponding outcomes arrive.

Total variation measures the actual surrogate mismatch and is the sharp event-transfer penalty.
Classifier regret converts a
standard learning objective into a ratio-error bound without access to the true ratio. Its conservatism
in our experiment comes from a global clipping parameter. Replacing $1/\gamma$ by a localized
control of $\EE_M[|\eta-h|/(1-h)]$ would tighten learned coverage bounds.

\section{Limitations}

Our experiment uses controlled tilts on a fixed pool of 192 physical transitions, and we compute
population classifier regret by finite summation. Our guarantees are one-step and marginal and
require shared conditional dynamics, target support within source support, and pre-outcome
weights; mechanism shift, including generated dynamics \parencite{jang2025dreamgen}, and joint
trajectory coverage require additional data or methods.

\section*{Disclosure}

We used a large language model for prose, document organization, review, and grammar assistance.

\FloatBarrier
\printbibliography[title={References}]

\appendix

\section{Transport distance and reweighting error}
\label{app:transport-separation}

A geometric distance can be small while the likelihood ratio oscillates at fine scale. The
separation persists under compact support, bounded densities, and bounded log-ratios.

\begin{proposition}[Bounded-density separation]
\label{prop:w2-counterexample}
Fix $a\in(0,1)$. Let $P_n$ be uniform on $[0,1]$. Partition the interval into $2n$ adjacent cells
of width $1/(2n)$, and let $Q_n$ have density $1+a$ on even cells and $1-a$ on odd cells. Then
\begin{align}
  \left\|\log\frac{\dd Q_n}{\dd P_n}\right\|_\infty
  &\le \max\{\log(1+a),-\log(1-a)\},\\
  \chi^2(Q_n\|P_n)&=a^2,
  &\TV(Q_n,P_n)&=\frac a2,\\
  W_1(P_n,Q_n)&\le\frac{a}{4n},
  &W_2(P_n,Q_n)^2&\le\frac{a}{8n^2}.
\end{align}
Consequently, no function $g(r)\to0$ can upper-bound $\chi^2(Q\|P)$ or $\TV(Q,P)$ from
$W_2(P,Q)\le r$ under either regularity class.
\end{proposition}

\begin{proof}
The ratio is $1\pm a$, so the log-ratio, $\chi^2$, and total-variation claims follow by direct
integration. We couple $P_n$ and $Q_n$ by leaving common mass fixed and moving the excess $a$-fraction
from each odd cell one cell to the left. The moved mass is $a/2$ and its distance is $1/(2n)$,
which gives the stated $W_1$ and squared-$W_2$ costs.
\end{proof}

\begin{figure}[!htbp]
  \centering
  \includegraphics[width=\textwidth]{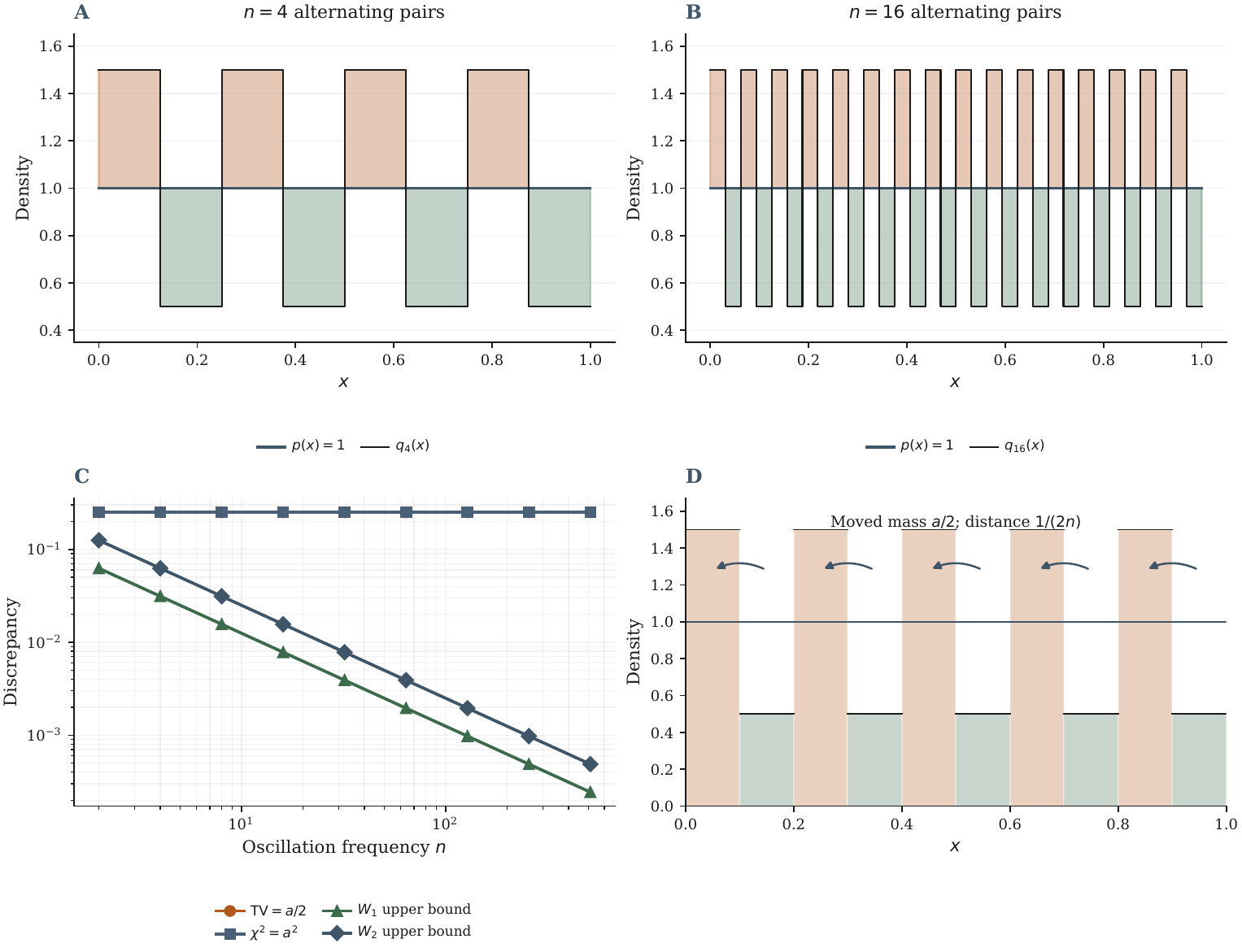}
  \caption{Fine-scale density oscillation separates transport geometry from reweighting error.
  Panels A--B show bounded densities on a fixed compact support. Panel C shows $W_1,W_2\to0$
  while total variation and $\chi^2$ remain fixed. Panel D gives the explicit local coupling.}
\label{fig:w2}
\end{figure}
\FloatBarrier

Wasserstein distance controls Lipschitz expectations and coverage events smoothed under score
anti-concentration \parencite{villani2009optimal}. It does not identify the likelihood ratio in
Proposition~\ref{prop:approx-coverage}.

\subsection{Score-event transport under anti-concentration}

Let $D_1(P,Q)$ be the dual Lipschitz discrepancy, equal to $W_1$ under the usual
Kantorovich--Rubinstein conditions \parencite{villani2009optimal}. If a score $r$ is $L$-Lipschitz
and
\[
  P(q-u<r\le q)\le\kappa u,
\]
then smoothing the indicator $\ind\{r\le q\}$ over a band of width $u$ yields
\[
  Q(r\le q)
  \ge P(r\le q)-\kappa u-\frac{L}{u}D_1(P,Q).
\]
Optimizing, we obtain a penalty $2\sqrt{\kappa L D_1(P,Q)}$ when the optimizer lies in the stated
anti-concentration range. This controls a particular score event; it does not control the density
ratio used by weighted conformal calibration.

\section{Proof details}
\label{app:proofs}

\subsection{Weighted conformal validity under the surrogate law}

Let $\mathcal D=(Z_1,\ldots,Z_n)$ and
$E_{\mathcal D}=\{(x,y):y\in\calC_{\widehat w}(x;\mathcal D)\}$. Because
$\dd\widehat Q/\dd P=\widehat w/Z_{\widehat w}$ and the weighted quantile is invariant to global
scaling, the covariate-shift conformal result of \textcite{tibshirani2019covariateshift} gives
\[
  \int \widehat Q(E_{\mathcal D})P^n(\dd\mathcal D)\ge1-\alpha.
\]
For each $\mathcal D$,
$Q(E_{\mathcal D})\ge\widehat Q(E_{\mathcal D})-\TV(Q,\widehat Q)$. Integration gives
\cref{eq:tv-transfer}. Normalizing $\widehat Q$ gives \cref{eq:l1-tv}. If $\widehat w$ is random
but fitted independently, the proof applies
conditionally on its training data.

\subsection{A finite-sample classifier interface}

Let $\mathcal H$ be a clipped posterior class. Suppose that, with probability at least
$1-\delta$,
\[
  \sup_{f\in\mathcal H}|R(f)-\widehat R_N(f)|\le u_N(\delta),
\]
that $\eta\in\mathcal H$, and that $h$ is a $\zeta$-approximate empirical risk minimizer. Then
\[
  \Delta_{\log}(h)\le 2u_N(\delta)+\zeta.
\]
Substitution into \cref{eq:classifier-bound} yields a high-probability ratio bound. For a
finite clipped class of size $K$, log loss is bounded by
$L_\gamma=\log(1/\gamma)$ and Hoeffding's inequality gives
\[
  u_N(\delta)=L_\gamma
  \sqrt{\frac{\log(2K/\delta)}{2N}}.
\]
We evaluate population regret exactly over $\Omega$.

\section{Algorithm}
\label{app:algorithm}

\begin{algorithm}[H]
\caption{One-step interventional conformal calibration}
\label{alg:method}
\begin{algorithmic}[1]
\Require fixed world model $m$; source calibration triples $(x_i,y_i)_{i=1}^n$; independent
source/target covariates for ratio fitting; level $\alpha$
\State identify $X$ so that the intended intervention preserves $P(Y\mid X)$
\State fit a domain posterior $h(x)=\Pr(D=1\mid X=x)$ with target-domain prior $\rho$
\State set $\widehat w(x)\gets\frac{1-\rho}{\rho}\frac{h(x)}{1-h(x)}$
\State compute source scores $R_i\gets r_m(x_i,y_i)$
\For{query $x$}
  \State form the weighted score measure in \cref{eq:weighted-measure}
  \State compute its $(1-\alpha)$ quantile $q_{\widehat w}(x)$
  \State return $\calC_{\widehat w}(x)=\{y:r_m(x,y)\le q_{\widehat w}(x)\}$
\EndFor
\end{algorithmic}
\end{algorithm}

\section{Experimental details}
\label{app:experiment}

\subsection{External assets and score construction}

We take the model and data loader from the official JEPA-WMs repository
\parencite{jepawms_repo}. The frozen DINO-WM uses a DINOv2 ViT-S/14 visual encoder and an
action-conditioned predictor at $224\times224$ resolution
\parencite{zhou2025dinowm,oquab2023dinov2}. Each eligible sample contains two context frames and
one observed next frame from the exterior camera. Pose differences produce the seven-dimensional
robot action representation used by the released model. We take the prediction score to be latent RMSE over
the predicted and observed next-frame patch representations.

The 192 transitions comprise 43 folding, 74 picking, and 75 pushing samples. Scores have mean
0.815, standard deviation 0.240, and range 0.376--1.458. The shift statistic and score have Pearson
correlation 0.800 and Spearman correlation 0.818 on this pool; we specified the shift statistic
without reading target frames or residuals.

\subsection{Sampling and uncertainty intervals}

For each $\lambda$, we draw independent source and target domain-classifier samples, a source
calibration sample of size 500, and a target evaluation sample of size 10,000. We fit logistic
regression with deterministic ridge regularization and repeat the full protocol 50 times. We form
coverage intervals in the main figures with the Student-$t$ distribution across repetitions; we
also compute Wilson intervals within each target sample, Bonferroni intervals across the full
preregistered grid, and an episode-cluster bootstrap that resamples the 15 physical episodes. 
\subsection{Complete results}

\input{table_full_results.tex}

\begin{figure}[!htbp]
  \centering
  \includegraphics[width=.72\textwidth]{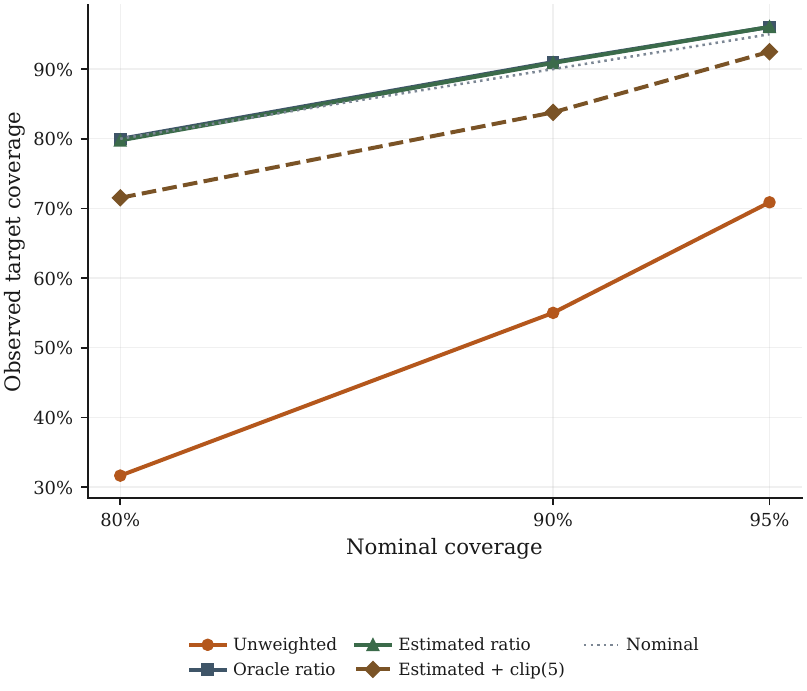}
  \caption{Observed versus nominal coverage at the strongest shift, $\lambda=1.5$. Oracle and
  learned-ratio calibration remain near the diagonal across the three preregistered levels;
  unweighted and clipped methods fall below it.}
  \label{fig:alpha-profile}
\end{figure}
\FloatBarrier

\section{Formalization}
\label{app:lean}

We encode in Lean 4 \parencite{demoura2021lean4} using Mathlib \parencite{mathlib2020} the
finite-support versions of Lemma~\ref{lem:exact-cancel}, Proposition~\ref{prop:classifier}, and
Proposition~\ref{prop:policy-ratio}, together with the arithmetic composition behind
Result~\ref{res:end-to-end}. The machine-checked content is the posterior-odds cancellation,
binary Pinsker, the finite weighted Cauchy--Schwarz step, the $(\rho\gamma)^{-1}$ classifier bound
and its balanced specialization, and the common-state propensity reweighting identity. We take the
weighted-conformal transfer inequality of \textcite{tibshirani2019covariateshift} as an explicit
premise and do not re-formalize it. The development is a single self-contained Lean package
available at the \href{https://github.com/noumenal-ai/conformal-realrobot-lab/tree/8446f05a094d42752569751945c96be70666b701/lab/lean}{commit-pinned repository},
which maps each result to its declaration; it builds against a commit-pinned
\href{https://github.com/leanprover-community/mathlib4}{\texttt{mathlib4}} revision and requires
no further dependency.

\end{document}

%% file: table_headline.tex
\begin{tabular}{@{}rrrrrrrr@{}}
\toprule
$\lambda$ & $\TV(P_\lambda,Q_\lambda)$ & Unweighted & Oracle & Estimated & Clip(5) & ESS$_{\mathrm{est}}$ & Radius$_{\mathrm{est}}$ \\
\midrule
0 & 0.0\% & 90.2\% & 90.2\% & 90.3\% & 90.3\% & 499.7 & 1.183 \\
0.5 & 24.6\% & 81.6\% & 90.5\% & 90.6\% & 90.6\% & 365.8 & 1.241 \\
1 & 46.4\% & 69.0\% & 91.8\% & 91.8\% & 90.8\% & 175.1 & 1.274 \\
1.5 & 63.7\% & 55.0\% & 91.0\% & 90.9\% & 83.8\% & 73.2 & 1.281 \\
\bottomrule
\end{tabular}

%% file: table_bounds.tex
\begin{tabular}{@{}rrrrrr@{}}
\toprule
$\lambda$ & Observed & Exact-TV bound & Regret bound & $\TV(Q,\widehat Q)$ & $\Delta_{\log}$ \\
\midrule
0 & 90.3\% & 89.2\% & 70.8\% & 0.8\% & 7.08e-05 \\
0.5 & 90.6\% & 89.3\% & 72.4\% & 0.7\% & 5.34e-05 \\
1 & 91.8\% & 89.2\% & 68.7\% & 0.8\% & 7.40e-05 \\
1.5 & 90.9\% & 89.1\% & 66.5\% & 0.9\% & 8.29e-05 \\
\bottomrule
\end{tabular}

%% file: table_full_results.tex
\begin{longtable}{@{}rrlrrrr@{}}
\caption{Complete target-law coverage results. Intervals are 95\% Student-$t$ intervals across 50 repetitions.}\label{tab:full-results}\\
\toprule
$\lambda$ & Nominal & Method & Coverage & 95\% interval & Radius & ESS \\
\midrule
\endfirsthead
\toprule
$\lambda$ & Nominal & Method & Coverage & 95\% interval & Radius & ESS \\
\midrule
\endhead
0 & 95.0\% & Unweighted & 95.3\% & [95.0, 95.6]\% & 1.271 & 500.0 \\
0 & 95.0\% & Oracle & 95.3\% & [95.0, 95.6]\% & 1.271 & 500.0 \\
0 & 95.0\% & Estimated & 95.4\% & [95.1, 95.7]\% & 1.272 & 499.7 \\
0 & 95.0\% & Estimated + clip(5) & 95.4\% & [95.1, 95.7]\% & 1.272 & 499.7 \\
\addlinespace[2pt]
0 & 90.0\% & Unweighted & 90.2\% & [89.8, 90.6]\% & 1.182 & 500.0 \\
0 & 90.0\% & Oracle & 90.2\% & [89.8, 90.6]\% & 1.182 & 500.0 \\
0 & 90.0\% & Estimated & 90.3\% & [89.8, 90.7]\% & 1.183 & 499.7 \\
0 & 90.0\% & Estimated + clip(5) & 90.3\% & [89.8, 90.7]\% & 1.183 & 499.7 \\
\addlinespace[2pt]
0 & 80.0\% & Unweighted & 80.2\% & [79.7, 80.7]\% & 0.996 & 500.0 \\
0 & 80.0\% & Oracle & 80.2\% & [79.7, 80.7]\% & 0.996 & 500.0 \\
0 & 80.0\% & Estimated & 80.2\% & [79.7, 80.7]\% & 0.997 & 499.7 \\
0 & 80.0\% & Estimated + clip(5) & 80.2\% & [79.7, 80.7]\% & 0.997 & 499.7 \\
\addlinespace[2pt]
\midrule
0.5 & 95.0\% & Unweighted & 89.5\% & [88.9, 90.1]\% & 1.232 & 500.0 \\
0.5 & 95.0\% & Oracle & 95.4\% & [94.9, 95.8]\% & 1.293 & 367.1 \\
0.5 & 95.0\% & Estimated & 95.4\% & [95.0, 95.9]\% & 1.293 & 365.8 \\
0.5 & 95.0\% & Estimated + clip(5) & 95.4\% & [95.0, 95.9]\% & 1.293 & 365.8 \\
\addlinespace[2pt]
0.5 & 90.0\% & Unweighted & 81.6\% & [81.0, 82.3]\% & 1.093 & 500.0 \\
0.5 & 90.0\% & Oracle & 90.5\% & [89.8, 91.2]\% & 1.240 & 367.1 \\
0.5 & 90.0\% & Estimated & 90.6\% & [89.9, 91.2]\% & 1.241 & 365.8 \\
0.5 & 90.0\% & Estimated + clip(5) & 90.6\% & [89.9, 91.2]\% & 1.241 & 365.8 \\
\addlinespace[2pt]
0.5 & 80.0\% & Unweighted & 67.0\% & [66.2, 67.8]\% & 0.942 & 500.0 \\
0.5 & 80.0\% & Oracle & 80.3\% & [79.5, 81.1]\% & 1.070 & 367.1 \\
0.5 & 80.0\% & Estimated & 80.4\% & [79.6, 81.2]\% & 1.070 & 365.8 \\
0.5 & 80.0\% & Estimated + clip(5) & 80.4\% & [79.6, 81.2]\% & 1.070 & 365.8 \\
\addlinespace[2pt]
\midrule
1 & 95.0\% & Unweighted & 81.3\% & [80.7, 81.8]\% & 1.166 & 500.0 \\
1 & 95.0\% & Oracle & 96.3\% & [95.6, 96.9]\% & 1.315 & 175.0 \\
1 & 95.0\% & Estimated & 96.3\% & [95.6, 96.9]\% & 1.315 & 175.1 \\
1 & 95.0\% & Estimated + clip(5) & 95.7\% & [95.1, 96.3]\% & 1.307 & 191.0 \\
\addlinespace[2pt]
1 & 90.0\% & Unweighted & 69.0\% & [68.1, 69.9]\% & 1.001 & 500.0 \\
1 & 90.0\% & Oracle & 91.8\% & [91.1, 92.6]\% & 1.274 & 175.0 \\
1 & 90.0\% & Estimated & 91.8\% & [91.1, 92.6]\% & 1.274 & 175.1 \\
1 & 90.0\% & Estimated + clip(5) & 90.8\% & [89.9, 91.7]\% & 1.266 & 191.0 \\
\addlinespace[2pt]
1 & 80.0\% & Unweighted & 47.7\% & [46.5, 48.9]\% & 0.895 & 500.0 \\
1 & 80.0\% & Oracle & 81.4\% & [80.2, 82.5]\% & 1.164 & 175.0 \\
1 & 80.0\% & Estimated & 81.4\% & [80.3, 82.5]\% & 1.164 & 175.1 \\
1 & 80.0\% & Estimated + clip(5) & 79.5\% & [78.5, 80.5]\% & 1.133 & 191.0 \\
\addlinespace[2pt]
\midrule
1.5 & 95.0\% & Unweighted & 70.9\% & [69.9, 71.8]\% & 1.077 & 500.0 \\
1.5 & 95.0\% & Oracle & 96.1\% & [95.2, 96.9]\% & 1.329 & 72.8 \\
1.5 & 95.0\% & Estimated & 96.1\% & [95.2, 96.9]\% & 1.329 & 73.2 \\
1.5 & 95.0\% & Estimated + clip(5) & 92.5\% & [91.3, 93.7]\% & 1.295 & 124.6 \\
\addlinespace[2pt]
1.5 & 90.0\% & Unweighted & 55.0\% & [54.2, 55.8]\% & 0.944 & 500.0 \\
1.5 & 90.0\% & Oracle & 91.0\% & [89.7, 92.4]\% & 1.281 & 72.8 \\
1.5 & 90.0\% & Estimated & 90.9\% & [89.5, 92.3]\% & 1.281 & 73.2 \\
1.5 & 90.0\% & Estimated + clip(5) & 83.8\% & [82.3, 85.3]\% & 1.234 & 124.6 \\
\addlinespace[2pt]
1.5 & 80.0\% & Unweighted & 31.6\% & [30.8, 32.5]\% & 0.846 & 500.0 \\
1.5 & 80.0\% & Oracle & 80.0\% & [78.0, 82.0]\% & 1.185 & 72.8 \\
1.5 & 80.0\% & Estimated & 79.8\% & [77.8, 81.7]\% & 1.183 & 73.2 \\
1.5 & 80.0\% & Estimated + clip(5) & 71.5\% & [70.2, 72.8]\% & 1.085 & 124.6 \\
\addlinespace[2pt]
\midrule
\bottomrule
\end{longtable}

%% file: references.bib
@article{ha2018worldmodels,
  author       = {Ha, David and Schmidhuber, J{\"u}rgen},
  title        = {World Models},
  journal      = {arXiv preprint arXiv:1803.10122},
  year         = {2018},
  eprint       = {1803.10122},
  archiveprefix= {arXiv},
  primaryclass = {cs.LG},
  url          = {https://arxiv.org/abs/1803.10122}
}

@inproceedings{hafner2019planet,
  author    = {Hafner, Danijar and Lillicrap, Timothy and Fischer, Ian and Villegas, Ruben and Ha, David and Lee, Honglak and Davidson, James},
  title     = {Learning Latent Dynamics for Planning from Pixels},
  booktitle = {Proceedings of the 36th International Conference on Machine Learning},
  series    = {Proceedings of Machine Learning Research},
  volume    = {97},
  pages     = {2555--2565},
  year      = {2019},
  publisher = {PMLR},
  url       = {https://proceedings.mlr.press/v97/hafner19a.html}
}

@inproceedings{hafner2020dreamer,
  author    = {Hafner, Danijar and Lillicrap, Timothy and Ba, Jimmy and Norouzi, Mohammad},
  title     = {Dream to Control: Learning Behaviors by Latent Imagination},
  booktitle = {International Conference on Learning Representations},
  year      = {2020},
  url       = {https://arxiv.org/abs/1912.01603}
}

@article{hafner2025dreamerv3,
  author  = {Hafner, Danijar and Pasukonis, Jurgis and Ba, Jimmy and Lillicrap, Timothy},
  title   = {Mastering Diverse Control Tasks Through World Models},
  journal = {Nature},
  volume  = {640},
  pages   = {647--653},
  year    = {2025},
  doi     = {10.1038/s41586-025-08744-2},
  url     = {https://www.nature.com/articles/s41586-025-08744-2}
}

@inproceedings{chua2018pets,
  author    = {Chua, Kurtland and Calandra, Roberto and McAllister, Rowan and Levine, Sergey},
  title     = {Deep Reinforcement Learning in a Handful of Trials Using Probabilistic Dynamics Models},
  booktitle = {Advances in Neural Information Processing Systems},
  volume    = {31},
  year      = {2018},
  url       = {https://proceedings.neurips.cc/paper_files/paper/2018/hash/3de568f8597b94bda53149c7d7f5958c-Abstract.html}
}

@inproceedings{micheli2023iris,
  author    = {Micheli, Vincent and Alonso, Eloi and Fleuret, Fran{\c{c}}ois},
  title     = {Transformers are Sample-Efficient World Models},
  booktitle = {International Conference on Learning Representations},
  year      = {2023},
  url       = {https://iclr.cc/virtual/2023/poster/11839}
}

@inproceedings{hansen2024tdmpc2,
  author    = {Hansen, Nick and Su, Hao and Wang, Xiaolong},
  title     = {{TD-MPC2}: Scalable, Robust World Models for Continuous Control},
  booktitle = {International Conference on Learning Representations},
  year      = {2024},
  url       = {https://proceedings.iclr.cc/paper_files/paper/2024/hash/cf73d57b6dcda32b293df7c2d5341f49-Abstract-Conference.html}
}

@inproceedings{wu2023daydreamer,
  author    = {Wu, Philipp and Escontrela, Alejandro and Hafner, Danijar and Abbeel, Pieter and Goldberg, Ken},
  title     = {DayDreamer: World Models for Physical Robot Learning},
  booktitle = {Proceedings of The 6th Conference on Robot Learning},
  series    = {Proceedings of Machine Learning Research},
  volume    = {205},
  pages     = {2226--2240},
  year      = {2023},
  publisher = {PMLR},
  url       = {https://proceedings.mlr.press/v205/wu23c.html}
}

@inproceedings{zhou2025dinowm,
  author    = {Zhou, Gaoyue and Pan, Hengkai and {LeCun}, Yann and Pinto, Lerrel},
  title     = {{DINO-WM}: World Models on Pre-trained Visual Features Enable Zero-shot Planning},
  booktitle = {Proceedings of the 42nd International Conference on Machine Learning},
  series    = {Proceedings of Machine Learning Research},
  volume    = {267},
  pages     = {79115--79135},
  year      = {2025},
  publisher = {PMLR},
  url       = {https://proceedings.mlr.press/v267/zhou25t.html}
}

@article{assran2025vjepa2,
  author       = {Assran, Mido and Bardes, Adrien and Fan, David and Garrido, Quentin and Howes, Russell and others},
  title        = {{V-JEPA 2}: Self-Supervised Video Models Enable Understanding, Prediction and Planning},
  journal      = {arXiv preprint arXiv:2506.09985},
  year         = {2025},
  eprint       = {2506.09985},
  archiveprefix= {arXiv},
  url          = {https://arxiv.org/abs/2506.09985}
}

@article{terver2025jepawms,
  author       = {Terver, Basile and Yang, Tsung-Yen and Ponce, Jean and Bardes, Adrien and {LeCun}, Yann},
  title        = {What Drives Success in Physical Planning with Joint-Embedding Predictive World Models?},
  journal      = {arXiv preprint arXiv:2512.24497},
  year         = {2025},
  eprint       = {2512.24497},
  archiveprefix= {arXiv},
  url          = {https://arxiv.org/abs/2512.24497}
}

@article{quevedo2025worldgym,
  author       = {Quevedo, Julian and Sharma, Ansh Kumar and Sun, Yixiang and Suryavanshi, Varad and Liang, Percy and Yang, Sherry},
  title        = {WorldGym: World Model as an Environment for Policy Evaluation},
  journal      = {arXiv preprint arXiv:2506.00613},
  year         = {2025},
  eprint       = {2506.00613},
  archiveprefix= {arXiv},
  url          = {https://arxiv.org/abs/2506.00613}
}

@article{khazatsky2024droid,
  author       = {Khazatsky, Alexander and Pertsch, Karl and Nair, Suraj and Balakrishna, Ashwin and Dasari, Sudeep and others},
  title        = {{DROID}: A Large-Scale In-The-Wild Robot Manipulation Dataset},
  journal      = {arXiv preprint arXiv:2403.12945},
  year         = {2024},
  eprint       = {2403.12945},
  archiveprefix= {arXiv},
  url          = {https://arxiv.org/abs/2403.12945}
}

@article{oquab2023dinov2,
  author       = {Oquab, Maxime and Darcet, Timoth{\'e}e and Moutakanni, Th{\'e}o and Vo, Huy and Szafraniec, Marc and others},
  title        = {{DINOv2}: Learning Robust Visual Features without Supervision},
  journal      = {arXiv preprint arXiv:2304.07193},
  year         = {2023},
  eprint       = {2304.07193},
  archiveprefix= {arXiv},
  url          = {https://arxiv.org/abs/2304.07193}
}

@inproceedings{richens2025worldmodels,
  author    = {Richens, Jonathan and Everitt, Tom and Abel, David},
  title     = {General Agents Need World Models},
  booktitle = {Proceedings of the 42nd International Conference on Machine Learning},
  series    = {Proceedings of Machine Learning Research},
  volume    = {267},
  pages     = {51659--51687},
  year      = {2025},
  publisher = {PMLR},
  url       = {https://proceedings.mlr.press/v267/richens25a.html}
}

@inproceedings{jang2025dreamgen,
  author    = {Jang, Joel and Ye, Seonghyeon and Lin, Zongyu and Xiang, Jiannan and Bjorck, Johan and others},
  title     = {DreamGen: Unlocking Generalization in Robot Learning through Video World Models},
  booktitle = {Proceedings of The 9th Conference on Robot Learning},
  series    = {Proceedings of Machine Learning Research},
  volume    = {305},
  pages     = {5170--5194},
  year      = {2025},
  publisher = {PMLR},
  url       = {https://proceedings.mlr.press/v305/jang25a.html}
}

@inproceedings{chandra2025diwa,
  author    = {Chandra, Akshay L. and Nematollahi, Iman and Huang, Chenguang and Welschehold, Tim and Burgard, Wolfram and Valada, Abhinav},
  title     = {{DiWA}: Diffusion Policy Adaptation with World Models},
  booktitle = {Proceedings of The 9th Conference on Robot Learning},
  series    = {Proceedings of Machine Learning Research},
  volume    = {305},
  pages     = {3378--3400},
  year      = {2025},
  publisher = {PMLR},
  url       = {https://proceedings.mlr.press/v305/chandra25a.html}
}

@inproceedings{seo2025latentsafety,
  author    = {Seo, Junwon and Nakamura, Kensuke and Bajcsy, Andrea},
  title     = {Uncertainty-aware Latent Safety Filters for Avoiding Out-of-Distribution Failures},
  booktitle = {Proceedings of The 9th Conference on Robot Learning},
  series    = {Proceedings of Machine Learning Research},
  volume    = {305},
  pages     = {4442--4472},
  year      = {2025},
  publisher = {PMLR},
  url       = {https://proceedings.mlr.press/v305/seo25a.html}
}

@book{vovk2005algorithmic,
  author    = {Vovk, Vladimir and Gammerman, Alexander and Shafer, Glenn},
  title     = {Algorithmic Learning in a Random World},
  publisher = {Springer},
  address   = {New York},
  year      = {2005},
  doi       = {10.1007/b106715},
  url       = {https://link.springer.com/book/10.1007/b106715}
}

@article{lei2018distributionfree,
  author  = {Lei, Jing and G'Sell, Max and Rinaldo, Alessandro and Tibshirani, Ryan J. and Wasserman, Larry},
  title   = {Distribution-Free Predictive Inference for Regression},
  journal = {Journal of the American Statistical Association},
  volume  = {113},
  number  = {523},
  pages   = {1094--1111},
  year    = {2018},
  doi     = {10.1080/01621459.2017.1307116},
  url     = {https://arxiv.org/abs/1604.04173}
}

@inproceedings{tibshirani2019covariateshift,
  author    = {Tibshirani, Ryan J. and Barber, Rina Foygel and Cand{\`e}s, Emmanuel and Ramdas, Aaditya},
  title     = {Conformal Prediction Under Covariate Shift},
  booktitle = {Advances in Neural Information Processing Systems},
  volume    = {32},
  year      = {2019},
  url       = {https://papers.neurips.cc/paper_files/paper/2019/hash/8fb21ee7a2207526da55a679f0332de2-Abstract.html}
}

@article{barber2023beyond,
  author  = {Barber, Rina Foygel and Cand{\`e}s, Emmanuel J. and Ramdas, Aaditya and Tibshirani, Ryan J.},
  title   = {Conformal Prediction Beyond Exchangeability},
  journal = {The Annals of Statistics},
  volume  = {51},
  number  = {2},
  pages   = {816--845},
  year    = {2023},
  doi     = {10.1214/23-AOS2276},
  url     = {https://arxiv.org/abs/2202.13415}
}

@article{lei2021counterfactual,
  author  = {Lei, Lihua and Cand{\`e}s, Emmanuel J.},
  title   = {Conformal Inference of Counterfactuals and Individual Treatment Effects},
  journal = {Journal of the Royal Statistical Society: Series B},
  volume  = {83},
  number  = {5},
  pages   = {911--938},
  year    = {2021},
  doi     = {10.1111/rssb.12445},
  url     = {https://academic.oup.com/jrsssb/article/83/5/911/7056131}
}

@inproceedings{gibbs2021adaptive,
  author    = {Gibbs, Isaac and Cand{\`e}s, Emmanuel},
  title     = {Adaptive Conformal Inference Under Distribution Shift},
  booktitle = {Advances in Neural Information Processing Systems},
  volume    = {34},
  year      = {2021},
  url       = {https://proceedings.neurips.cc/paper/2021/hash/0d441de75945e5acbc865406fc9a2559-Abstract.html}
}

@inproceedings{chee2024controllers,
  author    = {Chee, Kong Yao and Silva, Thales C. and Hsieh, M. Ani and Pappas, George J.},
  title     = {Uncertainty Quantification and Robustification of Model-Based Controllers Using Conformal Prediction},
  booktitle = {Proceedings of the 6th Annual Learning for Dynamics \& Control Conference},
  series    = {Proceedings of Machine Learning Research},
  volume    = {242},
  pages     = {528--540},
  year      = {2024},
  publisher = {PMLR},
  url       = {https://proceedings.mlr.press/v242/chee24a.html}
}

@inproceedings{zhou2024trajectories,
  author    = {Zhou, Yanfei and Lindemann, Lars and Sesia, Matteo},
  title     = {Conformalized Adaptive Forecasting of Heterogeneous Trajectories},
  booktitle = {Proceedings of the 41st International Conference on Machine Learning},
  series    = {Proceedings of Machine Learning Research},
  volume    = {235},
  pages     = {62002--62056},
  year      = {2024},
  publisher = {PMLR},
  url       = {https://proceedings.mlr.press/v235/zhou24l.html}
}

@article{cauchois2024robustvalidation,
  author  = {Cauchois, Maxime and Gupta, Suyash and Ali, Alnur and Duchi, John C.},
  title   = {Robust Validation: Confident Predictions Even When Distributions Shift},
  journal = {Journal of the American Statistical Association},
  year    = {2024},
  url     = {https://arxiv.org/abs/2008.04267}
}

@article{bickel2009covariateshift,
  author  = {Bickel, Steffen and Br{\"u}ckner, Michael and Scheffer, Tobias},
  title   = {Discriminative Learning Under Covariate Shift},
  journal = {Journal of Machine Learning Research},
  volume  = {10},
  pages   = {2137--2155},
  year    = {2009},
  url     = {https://jmlr.org/papers/v10/bickel09a.html}
}

@inproceedings{sugiyama2007direct,
  author    = {Sugiyama, Masashi and Nakajima, Shinichi and Kashima, Hisashi and Buenau, Paul V. and Kawanabe, Motoaki},
  title     = {Direct Importance Estimation with Model Selection and Its Application to Covariate Shift Adaptation},
  booktitle = {Advances in Neural Information Processing Systems},
  volume    = {20},
  year      = {2007},
  url       = {https://proceedings.neurips.cc/paper_files/paper/2007/hash/be83ab3ecd0db773eb2dc1b0a17836a1-Abstract.html}
}

@inproceedings{menon2016linking,
  author    = {Menon, Aditya and Ong, Cheng Soon},
  title     = {Linking Losses for Density Ratio and Class-Probability Estimation},
  booktitle = {Proceedings of The 33rd International Conference on Machine Learning},
  series    = {Proceedings of Machine Learning Research},
  volume    = {48},
  pages     = {304--313},
  year      = {2016},
  publisher = {PMLR},
  url       = {https://proceedings.mlr.press/v48/menon16.html}
}

@inproceedings{nachum2019dualdice,
  author    = {Nachum, Ofir and Chow, Yinlam and Dai, Bo and Li, Lihong},
  title     = {{DualDICE}: Behavior-Agnostic Estimation of Discounted Stationary Distribution Corrections},
  booktitle = {Advances in Neural Information Processing Systems},
  volume    = {32},
  year      = {2019},
  url       = {https://proceedings.neurips.cc/paper/2019/hash/cf9a242b70f45317ffd281241fa66502-Abstract.html}
}

@inproceedings{liu2018horizon,
  author    = {Liu, Qiang and Li, Lihong and Tang, Ziyang and Zhou, Dengyong},
  title     = {Breaking the Curse of Horizon: Infinite-Horizon Off-Policy Estimation},
  booktitle = {Advances in Neural Information Processing Systems},
  volume    = {31},
  year      = {2018},
  url       = {https://proceedings.neurips.cc/paper_files/paper/2018/hash/dda04f9d634145a9c68d5dfe53b21272-Abstract.html}
}

@inproceedings{jiang2016doubly,
  author    = {Jiang, Nan and Li, Lihong},
  title     = {Doubly Robust Off-policy Value Evaluation for Reinforcement Learning},
  booktitle = {Proceedings of The 33rd International Conference on Machine Learning},
  series    = {Proceedings of Machine Learning Research},
  volume    = {48},
  pages     = {652--661},
  year      = {2016},
  publisher = {PMLR},
  url       = {https://proceedings.mlr.press/v48/jiang16.html}
}

@book{pearl2009causality,
  author    = {Pearl, Judea},
  title     = {Causality: Models, Reasoning, and Inference},
  edition   = {2},
  publisher = {Cambridge University Press},
  year      = {2009},
  doi       = {10.1017/CBO9780511803161},
  url       = {https://www.cambridge.org/core/books/causality/B0046844FAE10CBF274D4ACBDAEB5F5B}
}

@book{villani2009optimal,
  author    = {Villani, C{\'e}dric},
  title     = {Optimal Transport: Old and New},
  series    = {Grundlehren der mathematischen Wissenschaften},
  volume    = {338},
  publisher = {Springer},
  year      = {2009},
  doi       = {10.1007/978-3-540-71050-9},
  url       = {https://link.springer.com/book/10.1007/978-3-540-71050-9}
}

@inproceedings{demoura2021lean4,
  author    = {de Moura, Leonardo and Ullrich, Sebastian},
  title     = {The Lean 4 Theorem Prover and Programming Language},
  booktitle = {Automated Deduction -- CADE 28},
  series    = {Lecture Notes in Computer Science},
  volume    = {12699},
  pages     = {625--635},
  publisher = {Springer},
  year      = {2021},
  doi       = {10.1007/978-3-030-79876-5_37},
  url       = {https://link.springer.com/chapter/10.1007/978-3-030-79876-5_37}
}

@inproceedings{mathlib2020,
  author    = {{The mathlib Community}},
  title     = {The Lean Mathematical Library},
  booktitle = {Proceedings of the 9th ACM SIGPLAN International Conference on Certified Programs and Proofs},
  pages     = {367--381},
  year      = {2020},
  doi       = {10.1145/3372885.3373824},
  url       = {https://arxiv.org/abs/1910.09336}
}

@software{jepawms_repo,
  author  = {{FAIR}},
  title   = {{JEPA-WMs}: Code, Data and Weights for Physical Planning with Joint-Embedding Predictive World Models},
  year    = {2026},
  url     = {https://github.com/facebookresearch/jepa-wms},
  note    = {Git commit 13cf1d9c\allowbreak 7e476f53\allowbreak c17714d2\allowbreak e0f1dc23\allowbreak 9a883ce0}
}

@software{dinov2_repo,
  author  = {{FAIR}},
  title   = {{DINOv2}: PyTorch Code and Models},
  year    = {2026},
  url     = {https://github.com/facebookresearch/dinov2},
  note    = {Git commit 7764ea0f\allowbreak 912e53c9\allowbreak 2e82eb78\allowbreak a2a1631e\allowbreak 92725fc8}
}
